\documentclass[letterpaper, 10 pt, conference]{ieeeconf}  

\IEEEoverridecommandlockouts                              

\usepackage{graphicx}      
\usepackage{caption}
\usepackage{subcaption}
\usepackage{amsmath}
\usepackage{amssymb}
\usepackage{xcolor}
\usepackage{booktabs}

\usepackage{import}

\newtheorem{lemma}{Lemma}

\newtheorem{remark}{Remark}

\usepackage{multirow}
\usepackage{comment}
\usepackage{algorithm}
\usepackage{algorithmic}

\usepackage{mathtools, cases}
\usepackage{cuted} 
\usepackage{placeins}
\usepackage{bm}

\makeatletter
\let\NAT@parse\undefined
\makeatother
\usepackage[hidelinks]{hyperref} 

\title{Switched Turn-based Adaptive Source Seeking Strategy using Estimation and Information-driven Direction of Improvement}

\author{Shubhra Banerjee$^{1}$ and Satadal Ghosh$^{2}$%
\thanks{$^{1}$Shubhra Banerjee is a Ph.D. Scholar in the Department of Aerospace Engineering, Indian Institute of Technology Madras, Chennai 600036, India.
{\tt\small ae21d022@smail.iitm.ac.in}}%
\thanks{$^{2}$Satadal Ghosh is an Associate Professor, Department of Aerospace Engineering, Indian Institute of Technology Madras, Chennai 600036, India.
{\tt\small satadal@iitm.ac.in}}%
}

\begin{document}

\maketitle
\thispagestyle{empty}
\pagestyle{empty}


\begin{abstract}
Source seeking is a fundamental problem in applications such as gas leak localization, radiation monitoring, and environmental surveillance, where the source location must be inferred from noisy scalar measurements. In agent-based settings, estimation and motion are tightly coupled: the trajectory influences measurement quality, while measurements refine the source estimate.

This paper proposes a switched turn-based adaptive source-seeking strategy that integrates Extended Kalman Filter (EKF) estimation with Fisher Information Matrix (FIM)-based direction selection within a loop-based motion framework. Measurements are accumulated over each loop, and direction updates are performed at loop boundaries using both estimation uncertainty and predicted information gain, enabling a principled balance between exploration and exploitation.

Theoretical analysis and simulation results for stationary and moving sources show that the proposed joint strategy achieves lower estimation error and more effective distance reduction compared to purely information-driven and estimate-driven approaches, particularly under dynamic conditions.
\end{abstract}


\section{Introduction}


Source seeking refers to the problem of estimating the unknown origin of a signal field from measurements by on-board sensors of unmanned vehicle(s) a.k.a. agent(s) at different locations and steering the agent toward the source region. The source location is not directly observable and needs to be inferred from noisy measurements collected during the motion of agent(s). This problem plays a critical role in several real-world applications, such as determining the deepest regions of a water body~\cite{burian1996gradient}, locating sources of pollutants~\cite{yungaicela2018UAV}, and radiation~\cite{mascarich2018radiation}. 

The problem of robotic source seeking has been addressed in literature following several methods that include control-theoretic ones, estimation and information-theoretic ones, and geometric ones. Among the control-theoretic methods, extremum-seeking control methods~\cite{ariyur2003real} have been extensively utilized, where small motion perturbations are used to extract gradient information from signal variations and guide the agent toward the source. Different implementations following this method include schemes that modulate both speed and heading~\cite{ghods2010speed} as well as approaches that regulate only the heading while maintaining constant speed~\cite{cochran2009nonholonomic}, with stochastic variants proposed to improve robustness under uncertainty~\cite{lin2017stochastic}. Besides, an adaptive feedback control-based source seeking scheme was presented in~\cite{greiff2021target} by adjusting control inputs using online source estimates to handle uncertainties, while a sliding mode control-based strategy has been presented in~\cite{matveev2011navigation} for robust convergence toward the source. 
While control-theoretic methods are mathematically well-grounded, their command structure often involves states that are hard to measure and estimate. Also, in some of the formulations, signum-based commands lead to very sharp turn in the robot's motion. These problems are mitigated by leveraging geometry-based strategies that focus on simple and implementable motion patterns. 

Two geometric strategies, named as 'Full Circle' and 'Half Circle', were presented in \cite{kashyap2017pursuing}, which inherently provided the opportunity of zigzagging for effective exploration, and also as the direction estimate to the source improved over nominal circular loops, it was exploited simultaneously. This notion was subsequently further advanced in formulating switched turn strategies, in which the agent follows a nominal circular loop unless a turn switching event occurs in that loop. The turn switching event was first characterized as passage through the maximum signal point of the loop in Maxima-Turn-Switching (MTS) framework \cite{upasana2019maxima}, \cite{upasana2019collaborative}, and crossing the local gradient estimate direction in Gradient-Direction-Turn-Switching (GDTS) \cite{sp2022gradient}. A local gradient estimate-based augmentation of MTS was also presented in \cite{kamthe2020gradient}. 
These strategies are computationally efficient and produce smooth trajectories suitable for real systems. However, in most of the control-theoretic methods except \cite{lin2017stochastic} and geometry-based methods uncertainties associated with scalar field, on-board sensing and estimation are usually not directly included.

In contrast, estimation-based and information-theoretic methods leverage underlying probabilistic and / or statistical properties of the uncertainties involved in source localization. An info-taxis scheme was presented in~\cite{vergassola2007infotaxis}, following which robot actions were chosen such that information about the source location gets improved, thus making them effective even when measurements are limited or unreliable. In \cite{rolf2020successive}, a biased sampling-based iterative motion planning of a robot was presented with major weightage on the updated confidence interval of source location estimate. EKF-based source location estimate and gradient-following on the trace of FIM were presented in \cite{zhang2021source}, \cite{zhang2023distributed} for a multi-robot source seeking problem. Particle filter was leveraged in \cite{hayes2002distributed,neumann2013gas} to explicitly account for uncertainty and nonlinear measurement. Trajectory with maximal information gain was planned for source localization using model predictive control \cite{atanasov2014information} and sampling-based method \cite{hollinger2014sampling}. While these approaches are statistically well-founded, they often command highly zigzagging robot trajectories, which are kinodynamically infeasible. 


In order to obviate the limitations of different methodologies, in this paper, a turn switching strategy is presented, in which the agent follows nominal circular loops, while turn switching event gets triggered as the agent crosses a specific direction in that loop. Three such directions are explored - one of them is toward the EKF-based source estimate, while the second one is toward a direction governed by the gradient of trace of FIM (indicating maximum information gain), and the last one being a adaptive linear combination of the first two directions. The adaptive weights facilitate balanced exploration and exploitation strategies during this source seeking. Thus, the amalgamation of geometric method with estimation and information-based methods is the salient feature of the source localization method presented in this paper. Note that leveraging the switched turn strategy over nominal circular loops in the agent's motion, in one hand, ensures kino-dynamic feasibility of the agent's motion as it progresses toward the source, and on the other hand, it also ensures an inherent feature of systematic zigzagging, thereby facilitating the exploration process. Moreover, the effective use of EKF-based source estimate direction and FIM-based direction of information gain in triggering the turn switching event in any nominal loop facilitates efficient handling the uncertainties involved in environment and sensing. Theoretical analysis and simulation results are presented to demonstrate the effectiveness of the proposed source seeking approach for both stationary and moving sources.

The remainder of this paper is organized as follows. 
Section~\ref{sec:problem_formulation} presents the problem formulation, including the source, agent, and measurement models. 
Section~\ref{sec:methodology} describes the proposed loop-based source-seeking strategy and different direction update mechanisms along with the theoretical analysis. 
Section~\ref{sec:results} provides simulation results and comparative analysis for different source motion. 
Finally, Section~\ref{sec:conclusion} provides a few concluding remarks.

\section{Problem Formulation}
\label{sec:problem_formulation}

\begin{figure}[h]
    \centering
    \includegraphics[width=0.7\columnwidth]{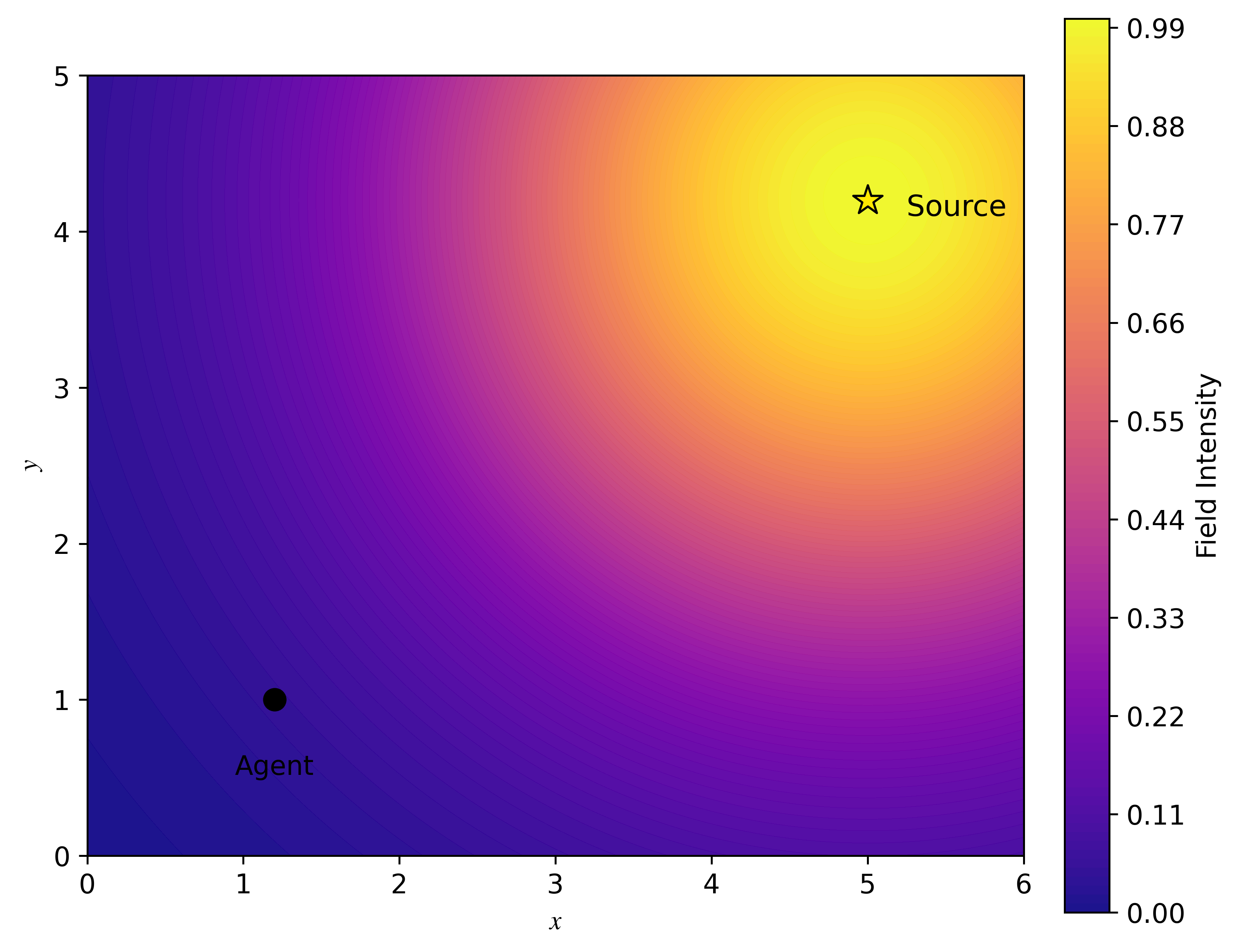}
    \caption{Planar source-seeking scenario. A sensor-mounted agent operates in a two-dimensional workspace and collects scalar field measurements generated by an unknown source.}
    \label{fig:problem_formulation}
\end{figure}

As illustrated in Fig. \ref{fig:problem_formulation}, This paper considers a single agent operating in a planar workspace to localize a signal source, which is prior unknown, using noisy scalar measurements by on-board sensors collected during the agent's motion.

\subsection{Source and Its Generated Scalar Field}
Let $\mathbf{s}(t)=[s_x(t)\:\:s_y(t)]^T \in \mathbb{R}^2$ denote the true source location at time $t$. The source may be stationary or time-varying. 
The source evolution is modeled in general form as
\begin{equation}
\dot{\mathbf{s}}(t) = f(\mathbf{s}(t), t) + \mathbf{w}(t),
\label{eq:source_motion}
\end{equation}

where $f(\cdot)$ defines the source kinematics and is assumed to be continuously differentiable, and $\mathbf{w}(t) \sim \mathcal{N}(\boldsymbol{\mu}, \mathbf{Q})$ represents Gaussian process noise with mean $\boldsymbol{\mu}$ and covariance $\mathbf{Q}$. The structure of $f(\cdot)$ is assumed known, while the true trajectory remains unknown.

The source is considered as generating a smooth scalar field, whose intensity attains a maximum at $\mathbf{s}(t)$ and decreases with distance from the source. Here, the spatial distribution of the signal field is modeled as a smooth symmetric and radially monotonically decreasing function centered at the source location, given by
\begin{equation}
h(\mathbf{x}(t), \mathbf{s}(t))
=
A \exp\left(-\frac{\|\mathbf{x}(t) - \mathbf{s}(t)\|^2}{2\sigma^2}\right),
\label{eq:field_model}
\end{equation}
where $A > 0$ denotes the source intensity and $\sigma > 0$ denotes the spatial spread of the field. 
\subsection{Agent Motion Model}

The agent position at time $t$ is denoted as $\mathbf{x}(t) =[x(t)\:\:y(t)]^T\in \mathbb{R}^2$ with heading angle $\psi(t) \in \mathbb{R}$.
The agent is considered to follow unicycle kinematics with forward speed $u(t)>0$ and turn rate $\omega(t)$.
\begin{align}
\dot{\mathbf{x}}(t) =
\begin{bmatrix}
u(t)\cos\psi(t) \\
u(t)\sin\psi(t)
\end{bmatrix};\quad \dot{\psi}(t) = \omega(t),
\label{eq:agent_motion}
\end{align}
\subsection{Measurement Model}

At time $t$, the agent obtains a scalar measurement
\begin{equation}
z(t) = h(\mathbf{x}(t), \mathbf{s}(t)) + v(t),
\label{eq:measurement}
\end{equation}
where $h(\cdot)$ as given in \eqref{eq:field_model} is a nonlinear function that models the spatial field generated by the source. The structure of $h(\cdot)$ is considered as known, however $\mathbf{s}(t)$ is unknown \textit{a-priori}.  And, $v(t) \sim \mathcal{N}(\mu_v, R)$ represents Gaussian measurement noise with mean $\mu_v$ and variance $R$.

\subsection{Problem Objective}

Given the source motion model in \eqref{eq:source_motion} and the measurement model in \eqref{eq:measurement}, the objective is to devise a kino-dynamically feasible motion strategy for the agent such that it steers the agent toward a small neighborhood about the source based on estimate of source location and information gain direction obtained from on-board sensor measurements.



\textcolor{black}{The agent motion is regulated using the current estimate $\hat{\mathbf{s}}$, the associated uncertainty, and an information measure derived from the Fisher Information Matrix, ensuring progressive improvement in localization performance.}

\section{Methodology}
\label{sec:methodology}
To address the problem, a loop-based exploration strategy is proposed, wherein the agent gathers spatially distributed measurements by traversing along constant-curvature trajectories. By accumulating measurements over each loop, high-frequency noise effects are attenuated, yielding more reliable directional information than instantaneous sampling. This underlying principle is also leveraged in existing methods such as MTS \cite{upasana2019maxima} and GDTS \cite{sp2022gradient}. At the end of each loop, the collected data is used to infer the source direction, which is subsequently used to switch the direction of the constant-curvature motion for the next loop. In existing approaches, this update is typically obtained via local gradient estimation of the signal field, which is highly sensitive to uncertainty. In contrast, the proposed method avoids explicit gradient estimation. The direction update is instead derived from the estimated source location provided by an EKF, together with a scalar information metric constructed from the measurements accumulated over the current loop.

\begin{figure}[t]
    \centering
    \includegraphics[width=0.5\linewidth]{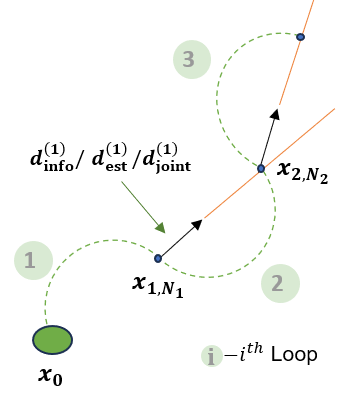}
    \caption{Loop-based source-seeking strategy with direction updated at the end of each loop (switching point).}
    \label{fig:method_overview}
\end{figure}




In the proposed strategy, the motion evolves as a sequence of constant-curvature loops indexed by \( i = 1, 2, \dots \). Let the \(i\)-th loop consist of \(N_i\) measurement steps, with the agent position at step \(k\) denoted by \( \mathbf{x}_{i,k} \), and the terminal (switching) point given by \( \mathbf{x}_{i,N_i} \). Within loop \(i\), for each measurement step \( k \in \{0,1,\dots,N_i\} \), the agent acquires scalar measurement $z_{i,k}$. These measurements along the loop are accumulated and used to compute a reference direction $\mathbf{d}^{(i)}$ for the subsequent loop.

The overall source-seeking strategy is outlined in the following section.



\subsection{Overall Strategy}

\textbf{Initial Loop:}
The agent initiates the exploration from an initial position $\mathbf{x}_0$ by executing a purely exploratory loop, as illustrated in Fig.~\ref{fig:method_overview}. During this loop, no prior directional information is available. The loop is therefore terminated at the point corresponding to the maximum measured signal, which is denoted as the first switching point $\mathbf{x}_{1,N_1}$. Using the measurements collected along the trajectory, a direction $\mathbf{d}^{(1)}$ is estimated, which determines the average direction of motion for the subsequent loop. 
In this paper, $\mathbf{d}^{(i)}$ is computed based on one of the following strategies. (i) an estimate-based strategy (see Section~\ref{sec:estimation_direction}),
(ii) a Fisher information-based strategy (see Section~\ref{sec:fim_direction}), or
(iii) a joint strategy (see Section~\ref{sec:joint_strategy_direction}), which balances estimation accuracy and information gain.

The agent then switches the direction of its motion accordingly and the loop counter $i$ is incremented.

\textbf{Subsequent Loops:}
Given the direction $\mathbf{d}^{(i-1)}$ at the previous switching point $\mathbf{x}_{i-1,N_{i-1}}$, the agent executes the next constant-curvature loop. The loop is terminated when the agent trajectory  
intersects the ray originating at $\mathbf{x}_{i-1,N_{i-1}}$ along direction $\mathbf{d}^{(i-1)}$. 
This intersection determines the next switching point $\mathbf{x}_{i,N_{i}}$, where a new direction $\mathbf{d}^{(i)}$ is estimated based on the accumulated measurements. The direction of agent's motion changes and the loop counter $i$ is incremented.



\textbf{Stopping Criterion:}
The iterative process continues until the measured signal at the switching points ceases to exhibit improvement across successive loops. Specifically, the algorithm is terminated when
\begin{equation}
z_{i,N_{i}} < z_{i-1,N_{i-1}},
\end{equation}
indicating convergence to the vicinity of the source.







\subsection{Estimation-Based Direction}
\label{sec:estimation_direction}

As described in the overall strategy, a direction $\mathbf{d}^{(i)}$ is computed at each switching point using the measurements accumulated over the preceding loop. In the estimation-based approach, this direction is denoted as $\mathbf{d}^{(i)}_{\mathrm{est}}$ and is derived from the estimate of the source location obtained from an EKF.

\subsubsection*{Extended Kalman Filter}


From \eqref{eq:source_motion} and \eqref{eq:measurement}, at each discrete measurement step $k$ within loop $i$, the source evolution model and measurement model are given as

\begin{align}
\mathbf{s}_{i,k} &= f(\mathbf{s}_{i,k-1}) + \mathbf{w}_{i,k-1}, 
\quad \mathbf{w}_{i,k-1} \sim \mathcal{N}(\mathbf{0}, \mathbf{Q}), \\
z_{i,k} &= h(\mathbf{x}_{i,k}, \mathbf{s}_{i,k}) + v_{i,k}, 
\quad v_{i,k} \sim \mathcal{N}(0, R).
\end{align}

Let $\hat{\mathbf{s}}_{i,k}$ and $\mathbf{P}_{i,k}$ denote the estimate of the source location and its associated covariance at step $k$ of loop $i$. Given the nonlinear process and measurement models, EKF proceeds through the following prediction and update steps.

\textit{Prediction:}
\begin{equation}
\hat{\mathbf{s}}_{i,k|k-1} = f(\hat{\mathbf{s}}_{i,k-1}),
 \end{equation}
\begin{equation}
\mathbf{P}_{i,k|k-1} = \mathbf{F}_{i,k-1} \mathbf{P}_{i,k-1} \mathbf{F}_{i,k-1}^\top + \mathbf{Q},
\end{equation}

\textit{Update:}
\begin{equation}
\mathbf{K}_{i,k} = \mathbf{P}_{i,k|k-1} \mathbf{H}_{i,k}^\top 
\left(\mathbf{H}_{i,k} \mathbf{P}_{i,k|k-1} \mathbf{H}_{i,k}^\top + \mathbf{R}\right)^{-1},
\end{equation}
\begin{equation}
\hat{\mathbf{s}}_{i,k} = \hat{\mathbf{s}}_{i,k|k-1} + \mathbf{K}_{i,k}
\left(z_{i,k} - h(\mathbf{X}_{i,k}, \hat{\mathbf{s}}_{i,k|k-1})\right),
\end{equation}
\begin{equation}
\mathbf{P}_{i,k} = (\mathbf{I} - \mathbf{K}_{i,k}\mathbf{H}_{i,k}) \mathbf{P}_{i,k|k-1},
\end{equation}

where $\mathbf{F}_{i,k-1}$ is the Jacobian of $f(\cdot)$ evaluated at $\hat{\mathbf{s}}_{i,k-1}$, and $\mathbf{H}_{i,k}$ is the Jacobian of $h(\cdot)$ with respect to the source state, evaluated at $(\mathbf{x}_{i,k}, \hat{\mathbf{s}}_{i,k|k-1})$, and $\mathbf{K}_{i,k}$ denotes the Kalman gain at step $k$ of loop $i$ and $\mathbf{x}_{i,k}$ represents the agent's position.

\subsubsection*{Estimate-Based Direction $\mathbf{d}^{(i)}_{\mathrm{est}}$}

At the switching point $\mathbf{x}_{i,N_i}$, the direction $\mathbf{d}^{(i)}_{\mathrm{est}}$ is defined based on the relative geometry between the agent the estimated source location.

Define the true relative position vector and its estimate as
\begin{equation}
\mathbf{r}_{i,N_i} \triangleq \mathbf{s}_{i,N_i} - \mathbf{x}_{i,N_i}, 
\quad
\hat{\mathbf{r}}_{i,N_i} \triangleq \hat{\mathbf{s}}_{i,N_i} - \mathbf{x}_{i,N_i}.
\end{equation}
The estimation error is given by
\begin{equation}
\mathbf{e}_{i,N_i} = \hat{\mathbf{s}}_{i,N_i} - \mathbf{s}_{i,N_i},
\end{equation}
which implies
\begin{equation}
\hat{\mathbf{r}}_{i,N_i} = \mathbf{r}_{i,N_i} + \mathbf{e}_{i,N_i}.
\end{equation}
The estimate-based direction for loop $i$ is then given by
\begin{equation}
\mathbf{d}_{\text{est}}^{(i)} =
\frac{\hat{\mathbf{r}}_{i,N_i}}{\|\hat{\mathbf{r}}_{i,N_i}\|}
=
\frac{\hat{\mathbf{s}}_{i,N_i} - \mathbf{x}_{i,N_i}}{\|\hat{\mathbf{s}}_{i,N_i} - \mathbf{x}_{i,N_i}\|}.
\label{eq:d_est}
\end{equation}
This $\mathbf{d}_{\mathrm{est}}^{(i)}$ will be used to terminate the subsequent loop and switch the agent's motion.

\subsubsection*{Results}
The estimate-based direction defined in \eqref{eq:d_est} depends explicitly on the EKF estimate at the switching points. This section characterize how $\mathbf{d}_{\text{est}}^{(i)}$ evolves across loops as the estimation improves.

\begin{lemma}[Loop-wise Improvement] \label{lem:Lemma_1}
    The estimate-based direction improves monotonically across loops 
    \begin{equation}
    \mathbb{E}\!\left[\|\mathbf{d}_{\text{est}}^{(i+1)} - \mathbf{d}^{*(i+1)}\|^2\right]
    <
    \mathbb{E}\!\left[\|\mathbf{d}_{\text{est}}^{(i)} - \mathbf{d}^{*(i)}\|^2\right].
    \end{equation}
    where $\mathbf{d}^{*(i)}$ denotes the true direction toward the source at the end of loop $i$, defined as
    \begin{equation}
    \mathbf{d}^{*(i)} = \frac{\mathbf{r}_{i,N_i}}{\|\mathbf{r}_{i,N_i}\|}, 
    \quad \mathbf{r}_{i,N_i} = \mathbf{s}_{i,N_i} - \mathbf{x}_{i,N_i}.
    \end{equation}
\end{lemma}
\begin{proof}
Using a first-order approximation of the normalized vector,
\begin{equation}
\mathbf{d}_{\text{est}}^{(i)} - \mathbf{d}^{*(i)}
\approx
\frac{1}{\|\mathbf{r}_{i,N_i}\|}\mathbf{P}_\perp \mathbf{e}_{i,N_i},
\end{equation}
where $\mathbf{P}_\perp = \mathbf{I} - \frac{\mathbf{r}_{i,N_i}\mathbf{r}_{i,N_i}^\top}{\|\mathbf{r}_{i,N_i}\|^2}$.

Thus,
\begin{equation}
\mathbb{E}\!\left[\|\mathbf{d}_{\text{est}}^{(i)} - \mathbf{d}^{*(i)}\|^2\right]
=
\frac{1}{\|\mathbf{r}_{i,N_i}\|^2}
\operatorname{tr}\!\left(\mathbf{P}_\perp \mathbf{P}_{i,N_i}\right).
\end{equation}

From the EKF information update,
\begin{equation}
\mathbf{P}_{i+1,N_{i+1}} \preceq \mathbf{P}_{i,N_i},
\end{equation}
Since $\mathbf{P}_\perp \succeq 0$ and the measurements provide non-zero information, the inequality is strict, which implies

\begin{equation}
\operatorname{tr}(\mathbf{P}_\perp \mathbf{P}_{i+1,N_{i+1}})
<
\operatorname{tr}(\mathbf{P}_\perp \mathbf{P}_{i,N_i}).
\end{equation}

Since $\|\mathbf{r}_{i+1,N_{i+1}}\| \le \|\mathbf{r}_{i,N_i}\|$ as shown in Section~\ref{sec:distance_reduction}, the result follows.
\end{proof}

\vspace{0.5em}

\begin{lemma}[Loop-wise Exponential Convergence] \label{lem:Lemma_2}
    Assume the EKF estimation error satisfies the mean-square exponential bound
    \begin{equation}
    \mathbb{E}\!\left[\|\mathbf{e}_k\|^2\right] \le C_0 \rho^k, 
    \quad 0 < \rho < 1,
    \end{equation}
    for all $k = 0,1,2,\dots$, where $C_0 > 0$ is a constant determined by the initial estimation error and noise statistics.
    
    Then, at the loop level, there exists a constant $C_1 \ge C_0$ such that
    \begin{equation}
    \mathbb{E}\!\left[\|\mathbf{e}_{i,N_i}\|^2\right] \le C_1 \rho^i,
    \end{equation}
    where $\mathbf{e}_{i,N_i} = \mathbf{e}_{k_i}$ denotes the estimation error at the switching point of loop $i$, and $k_i = \sum_{j=1}^{i} N_j$ is the cumulative time index.
    
    The direction error then satisfies
    \begin{equation}
    \mathbb{E}\!\left[\|\mathbf{d}_{\text{est}}^{(i)} - \mathbf{d}^{*(i)}\|^2\right]
    \le
    \frac{4C_1}{\|\mathbf{r}_{i,N_i}\|^2}\rho^i.
    \end{equation}
\end{lemma}
\begin{proof}
    From the definition of the switching index,
    \begin{equation}
    \mathbf{e}_{i,N_i} = \mathbf{e}_{k_i}, 
    \quad k_i = \sum_{j=1}^{i} N_j, \; N_j \ge 1.
    \end{equation}
    
    Using the EKF bound,
    \begin{equation}
    \mathbb{E}\!\left[\|\mathbf{e}_{i,N_i}\|^2\right]
    =
    \mathbb{E}\!\left[\|\mathbf{e}_{k_i}\|^2\right]
    \le
    C_0 \rho^{k_i}.
    \end{equation}
    
    Since $k_i \ge i$ and $0 < \rho < 1$, it follows that
    \begin{equation}
    \rho^{k_i} \le \rho^i,
    \end{equation}
    which implies
    \begin{equation}
    \mathbb{E}\!\left[\|\mathbf{e}_{i,N_i}\|^2\right] \le C_0 \rho^i.
    \end{equation}
    
    Thus, there exists a loop-level constant $C_1 \ge C_0$ such that
    \begin{equation}
    \mathbb{E}\!\left[\|\mathbf{e}_{i,N_i}\|^2\right] \le C_1 \rho^i.
    \end{equation}
    
    Next, using the bound on normalized vectors,
    \begin{equation}
    \|\mathbf{d}_{\text{est}}^{(i)} - \mathbf{d}^{*(i)}\|
    \le
    \frac{2\|\mathbf{e}_{i,N_i}\|}{\|\mathbf{r}_{i,N_i}\|},
    \end{equation}
    we obtain
    \begin{equation}
    \|\mathbf{d}_{\text{est}}^{(i)} - \mathbf{d}^{*(i)}\|^2
    \le
    \frac{4\|\mathbf{e}_{i,N_i}\|^2}{\|\mathbf{r}_{i,N_i}\|^2}.
    \end{equation}
    
    Assuming $\|\mathbf{r}_{i,N_i}\| \ge r_{\min} > 0$, the bound remains well-defined.
    
    Taking expectation and substituting the loop-level bound,
    \begin{equation}
    \mathbb{E}\!\left[\|\mathbf{d}_{\text{est}}^{(i)} - \mathbf{d}^{*(i)}\|^2\right]
    \le
    \frac{4C_1}{\|\mathbf{r}_{i,N_i}\|^2}\rho^i.
    \end{equation}
\end{proof}

Lemmas~\ref{lem:Lemma_1} and \ref{lem:Lemma_2} establish that $\mathbf{d}^{(i)}_{\mathrm{est}}$ becomes progressively more accurate across loops, with the direction error decreasing monotonically and exhibiting exponential convergence under standard EKF assumptions.

\subsection{Fisher Information-Based Direction}
\label{sec:fim_direction}

While the estimation-based strategy exploits the current source estimate to guide motion, it does not explicitly account for the information content of future measurements. To address this, an information-driven strategy is proposed, wherein the agent motion is regulated to maximize the information gained about the source location.

\subsubsection*{Fisher Information and Direction Definition}
At each measurement step, the FIM associated with the source estimate is given by \cite{simon2006optimal}
\begin{equation}
\mathbf{J}_{i,k} = \mathbf{H}_{i,k}^\top \mathbf{R}^{-1} \mathbf{H}_{i,k},
\end{equation}
where $\mathbf{H}_{i,k}$ is the measurement Jacobian evaluated at $(\mathbf{x}_{i,k}, \hat{\mathbf{s}}_{i,k})$. Over loop $i$, the accumulated information is $\mathbf{J}^{(i)} = \sum_{k=0}^{N_i} \mathbf{H}_{i,k}^\top \mathbf{R}^{-1} \mathbf{H}_{i,k}.$

At the switching point $\mathbf{x}_{i,N_i}$, the information-based direction is defined as 
\begin{equation}
\mathbf{d}_{\text{info}}^{(i)} =
\frac{
\nabla_{\mathbf{x}} \, \mathrm{tr}\!\left(\mathbf{J}(\mathbf{x}_{i,N_i}, \hat{\mathbf{s}}_{i,N_i})\right)
}{
\left\|
\nabla_{\mathbf{x}} \, \mathrm{tr}\!\left(\mathbf{J}(\mathbf{x}_{i,N_i}, \hat{\mathbf{s}}_{i,N_i})\right)
\right\|
}.
\end{equation}
The gradient is evaluated numerically at $\mathbf{x}_{i,N_i}$ using $\hat{\mathbf{s}}_{i,N_i}$ .

\subsection{Direction Estimation Using Joint Strategy} \label{sec:joint_strategy_direction}
To balance exploitation (estimation) and exploration (information gain), a joint direction is defined as
\begin{equation}
\mathbf{d}_{\text{joint}}^{(i)} =
(1 - \lambda_i)\mathbf{d}_{\text{info}}^{(i)} + \lambda_i \mathbf{d}_{\text{est}}^{(i)},
\quad 0 \le \lambda_i \le 1,
\label{eq:d_joint}
\end{equation}
where the weighting factor $\lambda_i$ is chosen as
\begin{equation}
\lambda_i = \frac{\operatorname{tr}(\mathbf{P}_{i,N_i})}{\operatorname{tr}(\mathbf{P}_{i,N_i}) + c},
\label{eq:lambda}
\end{equation}
with $c > 0$ being a tuning parameter.

This choice ensures that when the estimation uncertainty is large (i.e., $\operatorname{tr}(\mathbf{P}_{i,N_i})$ is large), the strategy prioritizes information gathering, whereas as the estimate becomes more accurate, the direction increasingly aligns with the estimate-based strategy.

The following result characterizes the advantage of incorporating Fisher information into the direction selection.

\begin{lemma} \label{lem:Lemma_3}
    The joint direction $\mathbf{d}_{\text{joint}}^{(i+1)}$ yields strictly smaller expected direction error compared to $\mathbf{d}_{\text{est}}^{(i+1)}$.
    \begin{equation}
    \mathbb{E}\!\left[\|\mathbf{d}_{\text{joint}}^{(i+1)} - \mathbf{d}^{*(i+1)}\|^2\right]
    <
    \mathbb{E}\!\left[\|\mathbf{d}_{\text{est}}^{(i+1)} - \mathbf{d}^{*(i+1)}\|^2\right].
    \end{equation}
\end{lemma}
\begin{proof}
From the estimate-based analysis,
\begin{equation}
\mathbb{E}\!\left[\|\mathbf{d}_{\text{est}}^{(i)} - \mathbf{d}^{*(i)}\|^2\right]
=
\frac{1}{\|\mathbf{r}_{i,N_i}\|^2}
\operatorname{tr}\!\left(\mathbf{P}_\perp \mathbf{P}_{i,N_i}\right).
\end{equation}

The EKF information update satisfies
\begin{equation}
\mathbf{P}_{i+1,N_{i+1}}^{-1}
=
\mathbf{P}_{i,N_i}^{-1} + \mathbf{J}^{(i+1)},
\end{equation}
where $\mathbf{J}^{(i+1)}$ depends on the chosen motion direction.

Since the information-driven direction maximizes the Fisher information, it yields
\begin{equation}
\mathbf{J}^{(i+1)}_{\text{info}} \succeq \mathbf{J}^{(i+1)}_{\text{est}}.
\end{equation}

Thus,
\begin{equation}
\mathbf{P}_{i+1,N_{i+1}}^{\text{info}}
\preceq
\mathbf{P}_{i+1,N_{i+1}}^{\text{est}},
\end{equation}
which implies
\begin{equation}
\operatorname{tr}\!\left(\mathbf{P}_\perp \mathbf{P}_{i+1,N_{i+1}}^{\text{info}}\right)
<
\operatorname{tr}\!\left(\mathbf{P}_\perp \mathbf{P}_{i+1,N_{i+1}}^{\text{est}}\right),
\end{equation}
and hence a smaller direction error. Since the joint direction includes a non-zero information-driven component, it achieves strictly greater information gain than the estimate-only strategy, thereby inheriting this improvement.
\end{proof}

\subsection{Convergence }
\label{sec:distance_reduction}

This subsection establishes that the proposed loop-based motion leads to progressive reduction in distance to the source, and provides insight into the role of different direction strategies.

\begin{lemma}[Loop-wise Distance Reduction] \label{lem:Lemma_4}
    The distance to the source decreases across successive switching points:
    \begin{equation}
    \|\mathbf{r}_{i+1,N_{i+1}}\| \le \|\mathbf{r}_{i,N_i}\|,
    \end{equation}
    where $\mathbf{r}_{i,N_i} = \mathbf{s}_{i,N_i} - \mathbf{x}_{i,N_i}$.    
\end{lemma}
\begin{proof}
    Define the displacement between successive switching points as
    \begin{equation}
    \Delta \mathbf{x}^{(i)} = \mathbf{x}_{i+1,N_{i+1}} - \mathbf{x}_{i,N_i},
    \end{equation}
    and the source variation as
    \begin{equation}
    \Delta \mathbf{s}^{(i)} = \mathbf{s}_{i+1,N_{i+1}} - \mathbf{s}_{i,N_i}.
    \end{equation}
    
    Then,
    \begin{equation}
    \mathbf{r}_{i+1,N_{i+1}} = \mathbf{r}_{i,N_i} + \Delta \mathbf{s}^{(i)} - \Delta \mathbf{x}^{(i)}.
    \end{equation}
    
    Expanding the squared norm,
    \begin{align}
    \|\mathbf{r}_{i+1,N_{i+1}}\|^2
    &= \|\mathbf{r}_{i,N_i}\|^2 
    - 2\mathbf{r}_{i,N_i}^\top \Delta \mathbf{x}^{(i)}
    + \|\Delta \mathbf{x}^{(i)}\|^2 \nonumber \\
    &\quad + 2\mathbf{r}_{i,N_i}^\top \Delta \mathbf{s}^{(i)}
    - 2(\Delta \mathbf{s}^{(i)})^\top \Delta \mathbf{x}^{(i)}
    + \|\Delta \mathbf{s}^{(i)}\|^2.
    \end{align}
    
    The loop is oriented along the estimate-based direction
    \begin{equation}
    \mathbf{d}_{\text{est}}^{(i)} =
    \frac{\mathbf{r}_{i,N_i} + \mathbf{e}_{i,N_i}}{\|\mathbf{r}_{i,N_i} + \mathbf{e}_{i,N_i}\|},
    \end{equation}
    which satisfies
    \begin{equation}
    \mathbf{r}_{i,N_i}^\top \mathbf{d}_{\text{est}}^{(i)} > 0,
    \end{equation}
    for sufficiently small estimation error.
    
    By construction of the loop motion,
    \begin{equation}
    \mathbf{d}_{\text{est}}^{(i)\top} \Delta \mathbf{x}^{(i)} > 0
    \quad \Rightarrow \quad
    \mathbf{r}_{i,N_i}^\top \Delta \mathbf{x}^{(i)} > 0.
    \end{equation}
    
    Thus, the dominant term $-2\mathbf{r}_{i,N_i}^\top \Delta \mathbf{x}^{(i)}$ is strictly negative.
    
    Assuming slow source variation,
    \begin{equation}
    \|\Delta \mathbf{s}^{(i)}\| \le \epsilon \|\mathbf{r}_{i,N_i}\|, \quad \epsilon \ll 1,
    \end{equation}
    the remaining terms are higher-order and do not dominate.
    
    Hence,
    \begin{equation}
    \|\mathbf{r}_{i+1,N_{i+1}}\|^2 < \|\mathbf{r}_{i,N_i}\|^2,
    \end{equation}
    which implies
    \begin{equation}
    \|\mathbf{r}_{i+1,N_{i+1}}\| \le \|\mathbf{r}_{i,N_i}\|.
    \end{equation}

\end{proof}

\begin{remark}[Convergence under Joint Strategy] \label{Remark_TotalConvergence}
From Lemma~\ref{lem:Lemma_3}, the joint strategy achieves smaller direction error than the estimate-based direction, resulting in improved alignment with the true source direction. From Lemma~\ref{lem:Lemma_4}, the estimate-based strategy ensures monotonic reduction in distance to the source and eventual convergence.

Since the joint strategy provides better directional alignment while inheriting the convergence property of the EKF, it yields a greater reduction in distance per loop. Hence, as the EKF converges, the joint strategy also guarantees convergence to the source, with faster progression compared to estimate-only and information-only strategies.
\end{remark}

\section{Simulation Results}
\label{sec:results}

In this section, the proposed algorithm is validated through numerical simulations performed in the Google Colab environment. The simulations were executed on a system equipped with an Intel(R) Core(TM) i5-10500 CPU (3.10 GHz) processor. All simulation experiments have been carried out using Python-based implementations. A Gaussian scalar field is considered as below.
\begin{equation}
h(\mathbf{x}_{i,k}, \mathbf{s}_{i,k})
=
A \exp\left(-\frac{\|\mathbf{x}_{i,k} - \mathbf{s}_{i,k}\|^2}{2\sigma^2}\right),
\label{eq:simulation_model}
\end{equation}
where $\mathbf{x}_{i,k}, \mathbf{s}_{i,k} \in \mathbb{R}^2$ denote the agent and source positions, respectively. 
The source intensity is $A = 12.0$ unit, and the spatial spread parameter is $\sigma = 14.0~\text{m}$. 
Measurement noise is Gaussian, $v_{i,k} \sim \mathcal{N}(0, R)$ with $R = 0.01$. The true initial source location is $\mathbf{s}_{0} = [23,\,23]^\top~\text{m}$, and the EKF is initialized at $\hat{\mathbf{s}}_{0} = [15,\,15]^\top~\text{m}$ with covariance 
$\mathbf{P}_{0} = 30\,\mathbf{I}_2$. The agent follows constant-curvature motion with constant forward speed 
$u = 0.6~\text{m/s}$ and constant turn rate 
$\omega = 0.6~\text{rad/s}$, yielding a nominal loop radius of $1.0$ m.

While considering the motion regulation strategies, three different directions are considered for the initiation of turn switching at every nominal loop: 
(i) \textbf{FIM-only}, where the direction $\mathbf{d}_{\text{info}}^{(i)}$ is selected to maximize the information gain;
(ii) \textbf{Estimate-only}, where the direction $\mathbf{d}_{\text{est}}^{(i)}$ is selected towards the estimated source location; and, (iii) \textbf{joint strategy}, where the direction $\mathbf{d}_{\text{joint}}^{(i)}$ in \eqref{eq:d_joint} combines $\mathbf{d}_{\text{info}}^{(i)}$ and $\mathbf{d}_{\text{est}}^{(i)}$ using the adaptive weight $\lambda_i$ defined in \eqref{eq:lambda} that facilitates a balance between exploitation and exploration during the source localization process.

Three types of source motion are considered: 
(1) stationary source at $\mathbf{s}_0$; 
(2) straight-line motion with velocity $v_{\text{source}} = 0.10~\text{m/s}$ along the $x$-direction; 
(3) circular motion with $v_{\text{source}} = 0.10~\text{m/s}$ and angular rate $\omega_{\text{source}} = 0.01~\text{rad/s}$, giving a turn radius of $R$ as
$R_{\text{source}} = 10~\text{m}$.

\begin{figure*}[t]
\centering

\begin{subfigure}{0.32\textwidth}
    \centering
    \includegraphics[width=0.8\linewidth]{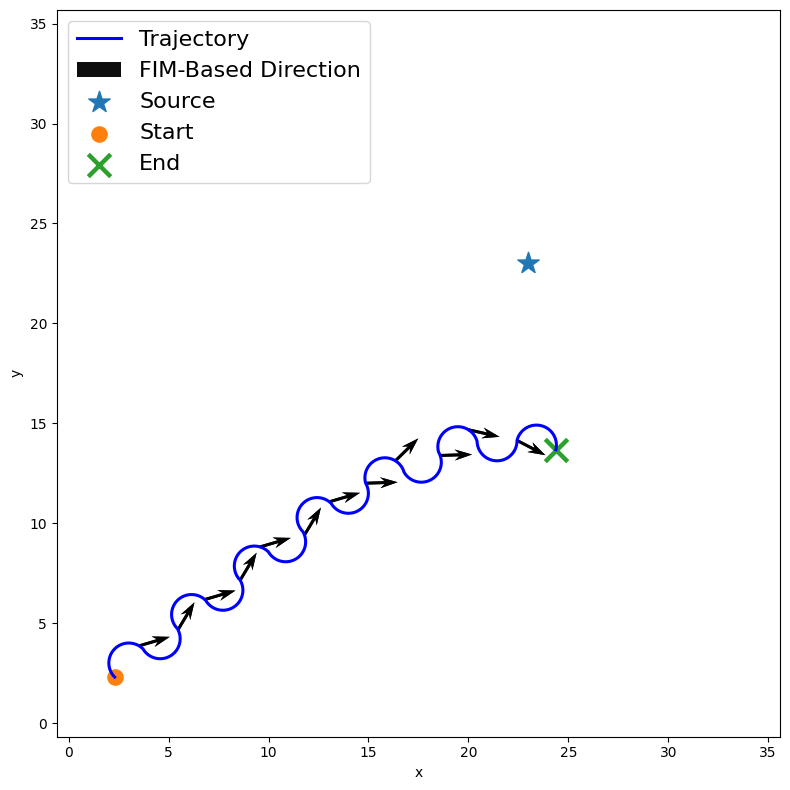}
    \caption{Stationary source — $\mathbf{d}_{\text{info}}^{(i)}$}
\end{subfigure}
\hfill
\begin{subfigure}{0.32\textwidth}
    \centering
    \includegraphics[width=0.8\linewidth]{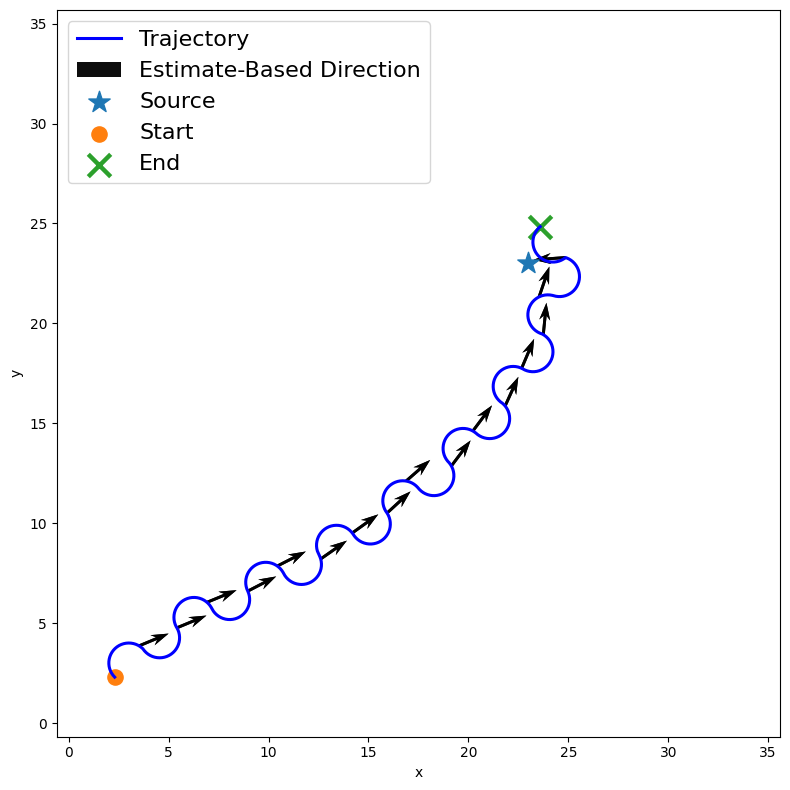}
    \caption{Stationary source — $\mathbf{d}_{\text{est}}^{(i)}$}
\end{subfigure}
\hfill
\begin{subfigure}{0.32\textwidth}
    \centering
    \includegraphics[width=0.8\linewidth]{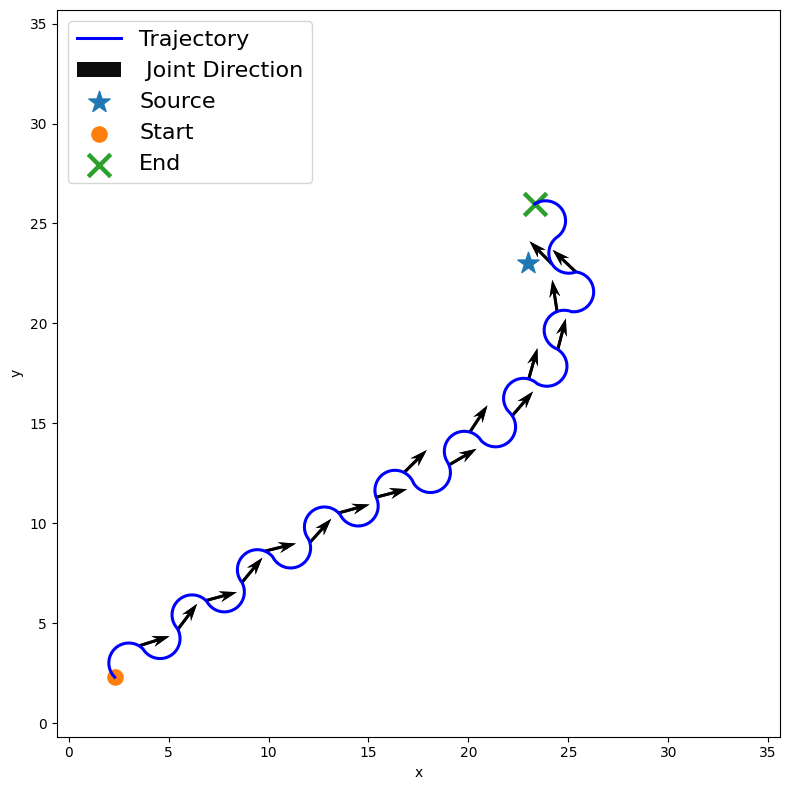}
    \caption{Stationary source — $\mathbf{d}_{\text{joint}}^{(i)}$}
\end{subfigure}

\vspace{0.35cm}

\begin{subfigure}{0.32\textwidth}
    \centering
    \includegraphics[width=0.8\linewidth]{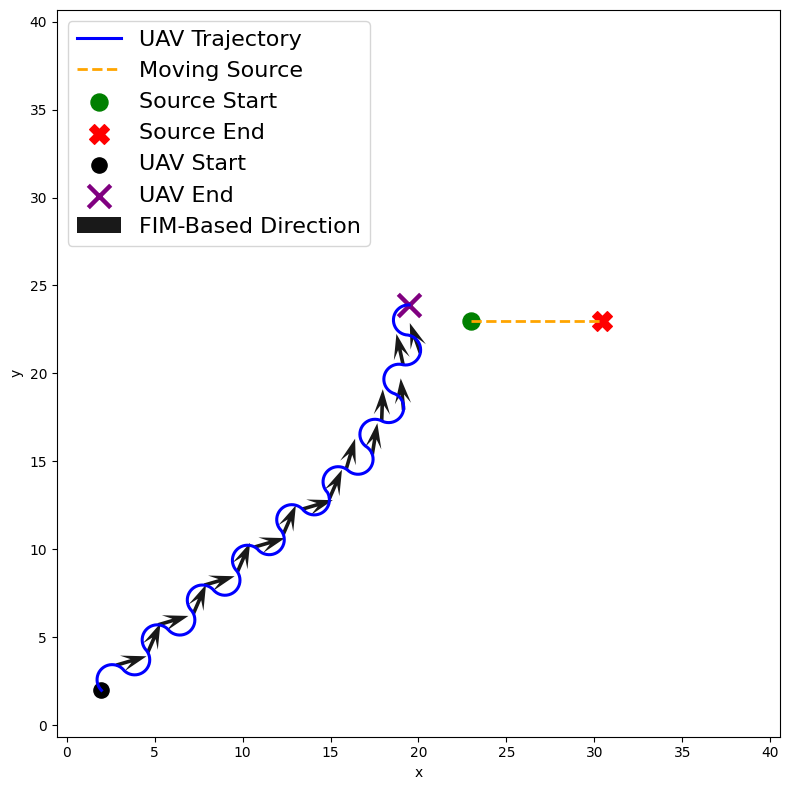}
    \caption{Non-maneuvering source — $\mathbf{d}_{\text{info}}^{(i)}$}
\end{subfigure}
\hfill
\begin{subfigure}{0.32\textwidth}
    \centering
    \includegraphics[width=0.8\linewidth]{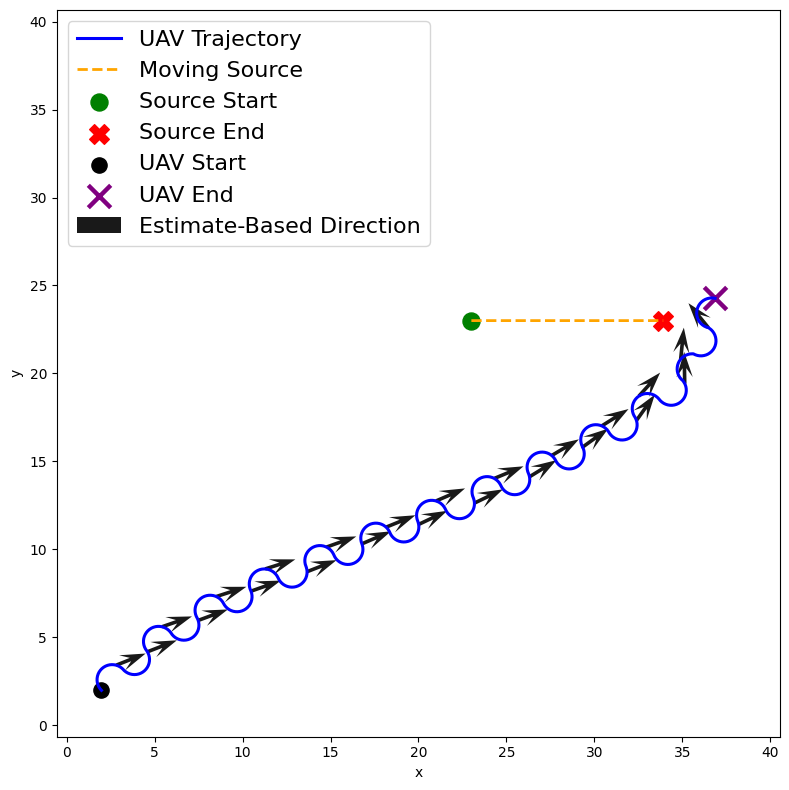}
    \caption{Non-maneuvering source — $\mathbf{d}_{\text{est}}^{(i)}$}
\end{subfigure}
\hfill
\begin{subfigure}{0.32\textwidth}
    \centering
    \includegraphics[width=0.8\linewidth]{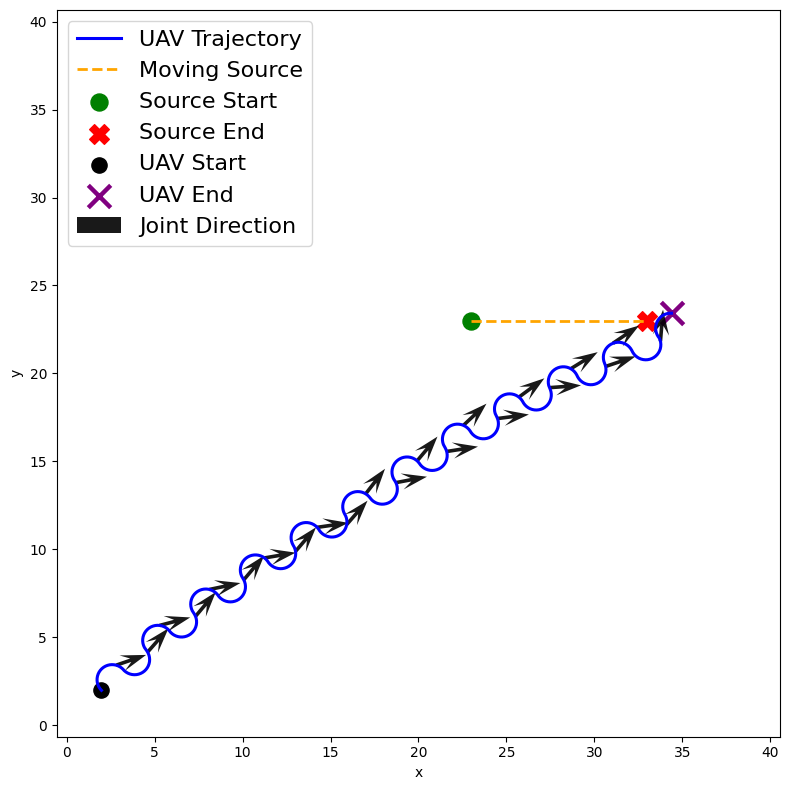}
    \caption{Non-maneuvering source — $\mathbf{d}_{\text{joint}}^{(i)}$}
\end{subfigure}

\vspace{0.35cm}

\begin{subfigure}{0.32\textwidth}
    \centering
    \includegraphics[width=0.8\linewidth]{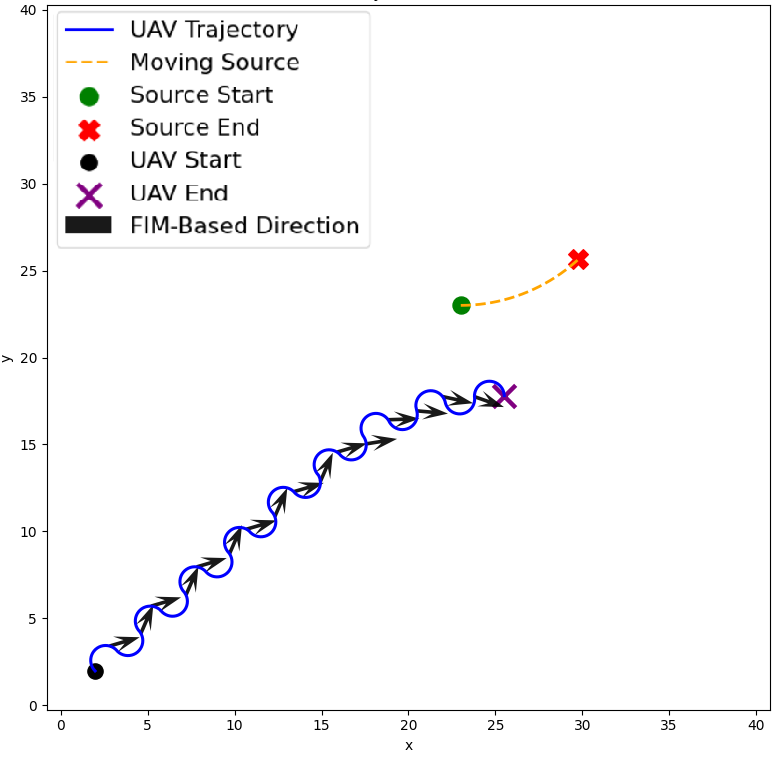}
    \caption{Maneuvering source — $\mathbf{d}_{\text{info}}^{(i)}$}
\end{subfigure}
\hfill
\begin{subfigure}{0.32\textwidth}
    \centering
    \includegraphics[width=0.8\linewidth]{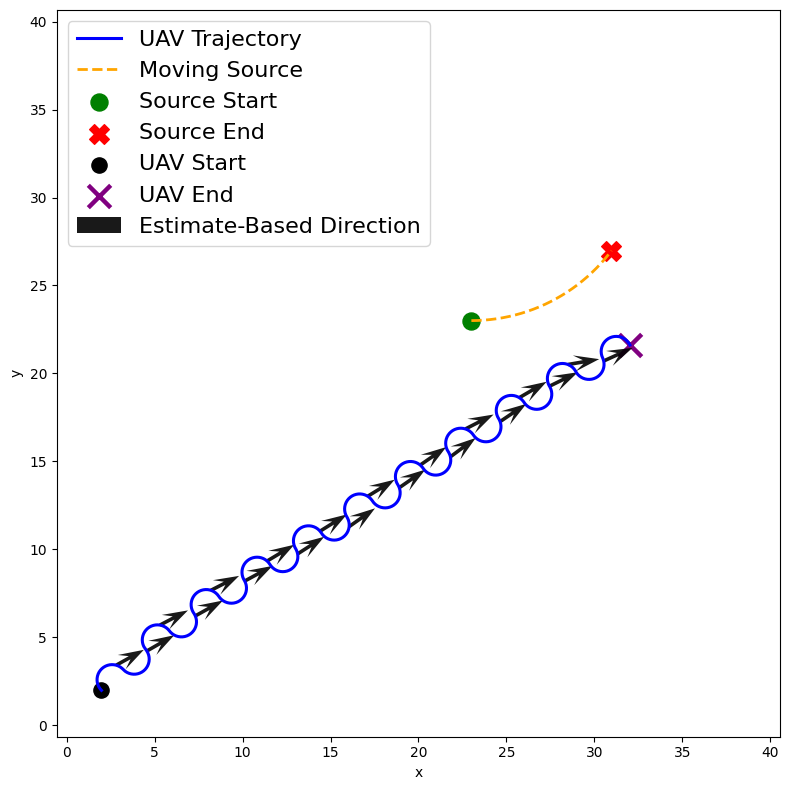}
    \caption{Maneuvering source — $\mathbf{d}_{\text{est}}^{(i)}$}
\end{subfigure}
\hfill
\begin{subfigure}{0.32\textwidth}
    \centering
    \includegraphics[width=0.8\linewidth]{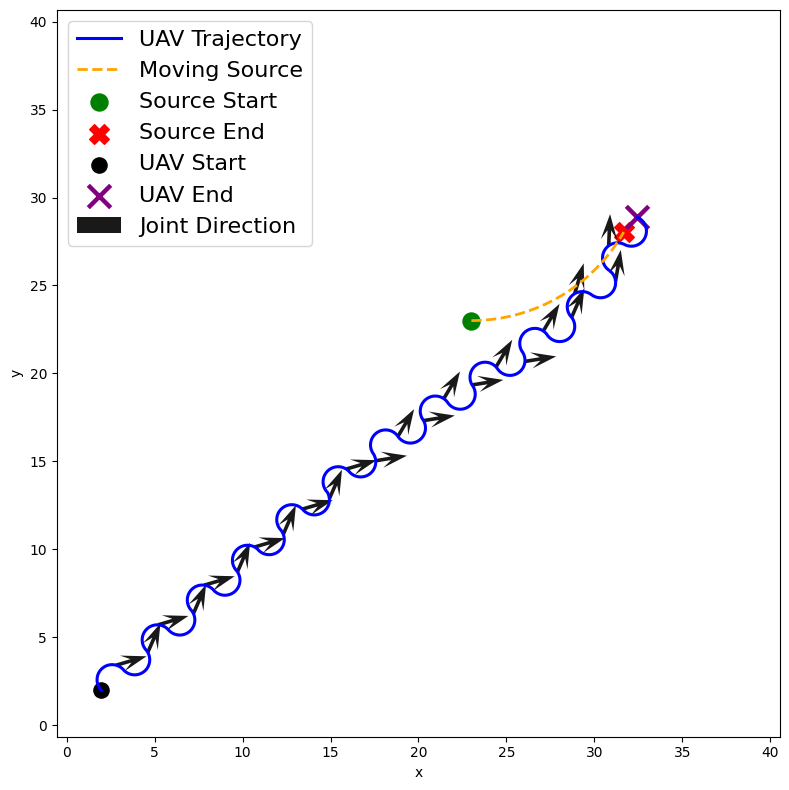}
    \caption{Maneuvering source — $\mathbf{d}_{\text{joint}}^{(i)}$}
\end{subfigure}

\caption{Trajectory comparison across three source dynamics (rows) and motion regulation strategies (columns). 
The information-driven $\mathbf{d}_{\text{info}}^{(i)}$ promotes exploration, while the estimate-driven $\mathbf{d}_{\text{est}}^{(i)}$ exploits the current source estimate. 
The joint strategy $\mathbf{d}_{\text{joint}}^{(i)}$ balances both, resulting in improved convergence, reduced tracking lag, and more stable trajectories across stationary and dynamic scenarios.}
\label{fig:trajectory_3x3}
\end{figure*}

\subsection{Trajectory Comparison}
\label{subsec:traj_comparison}

It is evident from Fig. \ref{fig:trajectory_3x3} that the turn switching strategy based on both $\mathbf{d}_{\text{info}}^{(i)}$, $\mathbf{d}_{\text{est}}^{(i)}$ and $\mathbf{d}_{\text{joint}}^{(i)}$ initially for a long duration steers the agent toward the source in cases of both stationary and dynamic source scenarios. However, as the information saturates after the agent's passage through the circular loops during that period, $\mathbf{d}_{\text{info}}^{(i)}$ starts oscillating with large deviation, thus resulting in poor performance in convergence toward the source at the end. On the other hand, as the source estimate improves over the circular loops, $\mathbf{d}_{\text{est}}^{(i)}$ gradually converges toward the source, thus leading the agent following $\mathbf{d}_{\text{est}}^{(i)}$-based turn switching strategy also converge sufficiently close to the source at the end. However, as $\mathbf{d}_{\text{est}}^{(i)}$ is a linear combination of both $\mathbf{d}_{\text{info}}^{(i)}$ and $\mathbf{d}_{\text{est}}^{(i)}$, it also oscillates, but with a less deviation w.r.t. $\mathbf{d}_{\text{info}}^{(i)}$. Moreover, as the weightage of $\mathbf{d}_{\text{est}}^{(i)}$ is more significant on $\mathbf{d}_{\text{joint}}^{(i)}$ at the end, it also finally converges toward the source with mild oscillatory pattern. This leads the agent driven by $\mathbf{d}_{\text{joint}}^{(i)}$-based turn switching strategy approach sufficiently close to the source.

In case of dynamic source scenarios, this pattern remains same for $\mathbf{d}_{\text{info}}^{(i)}$ and $\mathbf{d}_{\text{est}}^{(i)}$-based turn switching strategies. However, unlike stationary source scenario, the information about the source and its generated scalar field do not get saturated in dynamic source cases. Thus, the adaptive linear combination of $\mathbf{d}_{\text{info}}^{(i)}$ and $\mathbf{d}_{\text{est}}^{(i)}$ in $\mathbf{d}_{\text{joint}}^{(i)}$ facilitates in simultaneously exploring and exploiting even at the second half of the source localization mission, thus improving the convergence performance of the agent to the dynamic source's closer proximity over individual $\mathbf{d}_{\text{info}}^{(i)}$ and $\mathbf{d}_{\text{est}}^{(i)}$-based turn switching strategies.




\subsection{Estimation Error Evolution}
\label{subsec:error_analysis}

\begin{figure*}[t]
\centering

\begin{subfigure}{0.32\textwidth}
    \centering
    \includegraphics[width=0.8\linewidth]{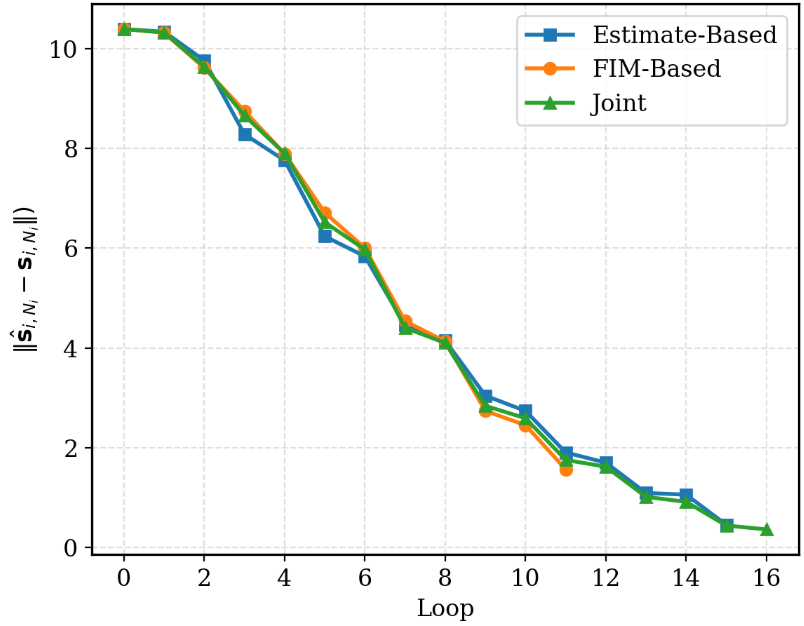}
    \caption{Stationary source}
\end{subfigure}
\hfill
\begin{subfigure}{0.32\textwidth}
    \centering
    \includegraphics[width=0.8\linewidth]{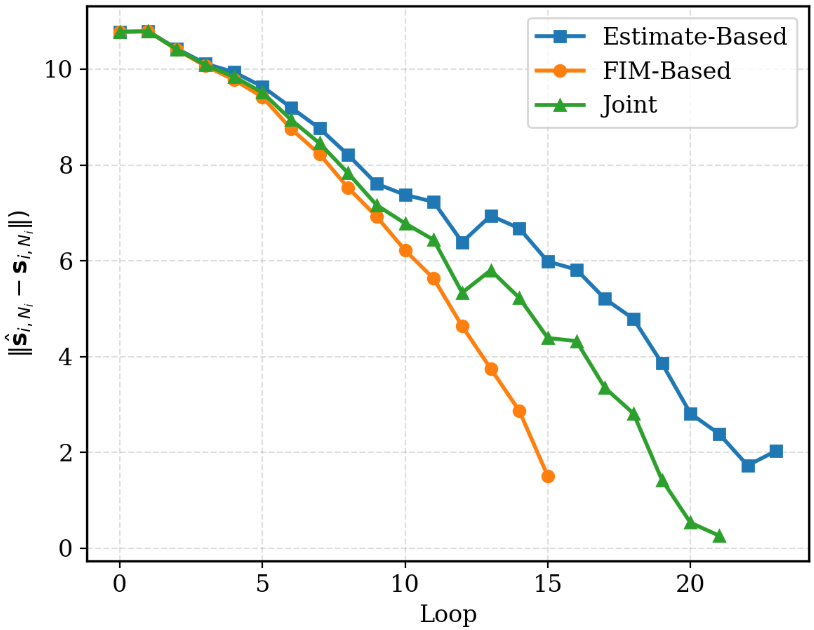}
    \caption{Non-maneuvering source}
\end{subfigure}
\hfill
\begin{subfigure}{0.32\textwidth}
    \centering
    \includegraphics[width=0.8\linewidth]{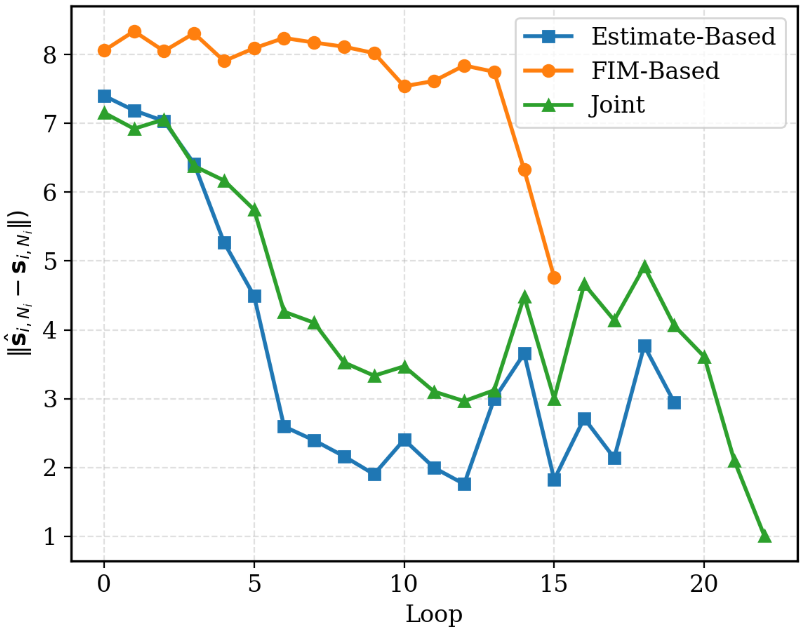}
    \caption{Maneuvering source}
\end{subfigure}

\vspace{0.1cm}

\caption{Estimation error $\left\lVert \hat{\mathbf{s}}_{i,N_i} - \mathbf{s}_{i,N_i} \right\rVert$ evaluated at switching points for stationary, non-maneuvering  and maneuvering source source dynamics. The joint strategy achieves faster convergence and lower terminal error across all cases.}
\label{fig:error_compare}
\end{figure*}

As observed in Fig. \ref{fig:error_compare}, across all source dynamics, the joint strategy achieves the lowest terminal estimation error. 
For the stationary case, the error reduces to $0.366~\mathrm{m}$, compared to $1.569~\mathrm{m}$ for $\mathbf{d}_{\text{info}}^{(i)}$ and $0.457~\mathrm{m}$ for $\mathbf{d}_{\text{est}}^{(i)}$. 
Under straight-line motion, it attains $0.265~\mathrm{m}$, outperforming $\mathbf{d}_{\text{info}}^{(i)}$ ($1.507~\mathrm{m}$) and $\mathbf{d}_{\text{est}}^{(i)}$ ($2.033~\mathrm{m}$). 
For circular motion, the joint strategy achieves $1.01~\mathrm{m}$, whereas $\mathbf{d}_{\text{info}}^{(i)}$ and $\mathbf{d}_{\text{est}}^{(i)}$ result in higher residual errors. 
These results indicate that the joint strategy improves estimation accuracy across both stationary and dynamic source scenarios.

\subsection{Distance-to-Source Evolution}

\begin{figure*}[t]
\centering

\begin{subfigure}{0.32\textwidth}
    \centering
    \includegraphics[width=0.8\linewidth]{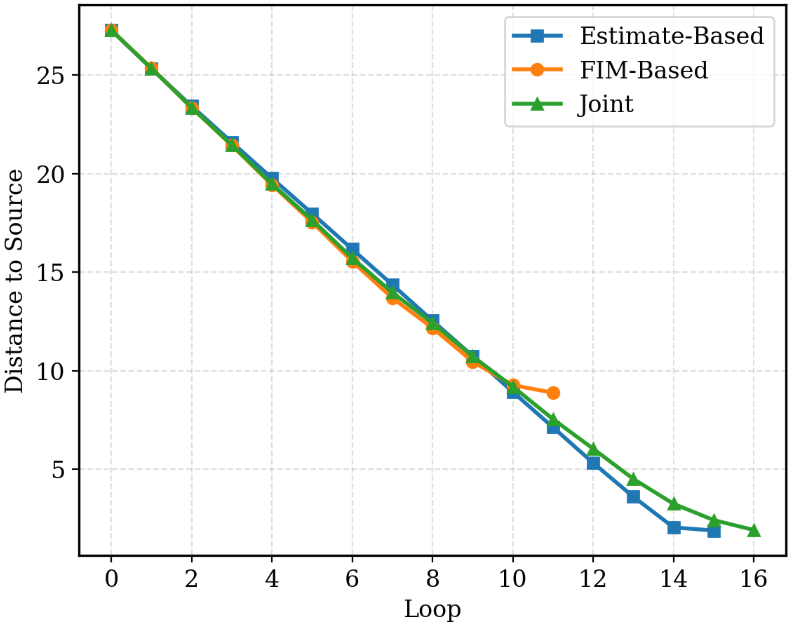}
    \caption{Stationary Source}
\end{subfigure}
\hfill
\begin{subfigure}{0.32\textwidth}
    \centering
    \includegraphics[width=0.8\linewidth]{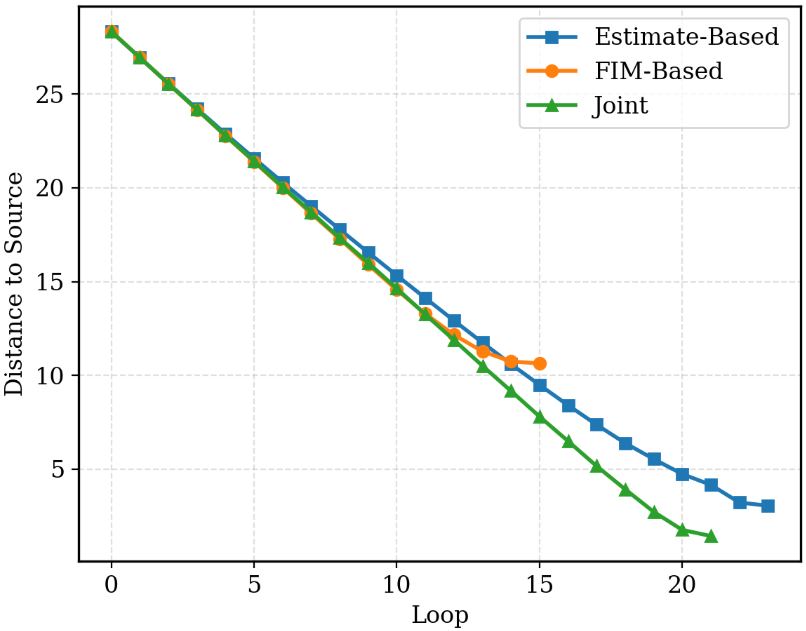}
    \caption{Non-maneuvering source}
\end{subfigure}
\hfill
\begin{subfigure}{0.32\textwidth}
    \centering
    \includegraphics[width=0.8\linewidth]{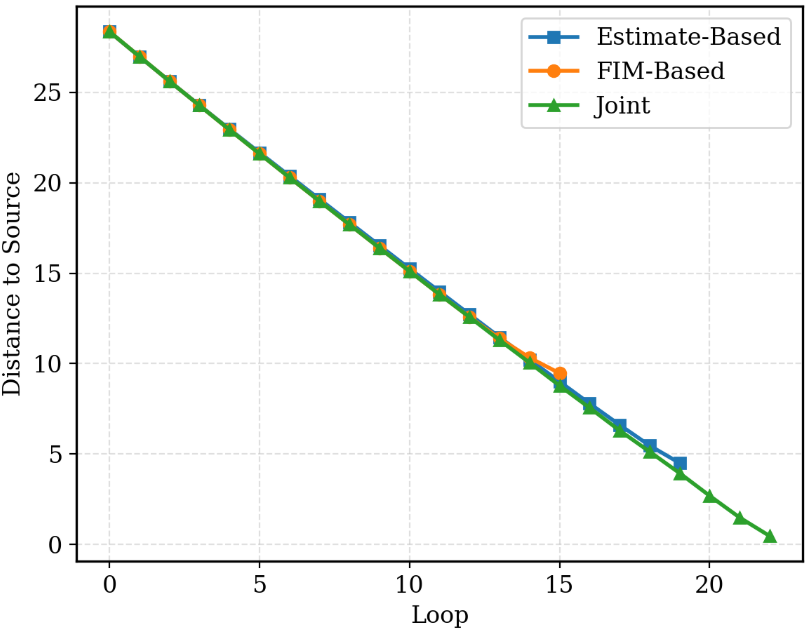}
    \caption{Maneuvering source}
\end{subfigure}

\vspace{0.1cm}

\caption{Distance to source $\|\mathbf{r}_{i,N_i}\|$ versus loop index $i$ for stationary, non-maneuvering and maneuvering source source dynamics. 
The joint strategy $\mathbf{d}_{\text{joint}}^{(i)}$ consistently yields faster convergence and reduced tracking lag compared to the information-driven $\mathbf{d}_{\text{info}}^{(i)}$ and estimate-driven $\mathbf{d}_{\text{est}}^{(i)}$ strategies.}
\label{fig:distance_compare}
\end{figure*}

Fig.~\ref{fig:distance_compare} compares the agent--source Euclidean distance (m) for all strategies. 
Across all source dynamics, the joint strategy $\mathbf{d}_{\text{joint}}^{(i)}$ shows a consistently faster reduction in distance. For all source scenarios, the $\mathbf{d}_{\text{info}}^{(i)}$-based turn switching strategy suffers at the end in terms of convergence, while both the $\mathbf{d}_{\text{est}}^{(i)}$-based and $\mathbf{d}_{\text{joint}}^{(i)}$-based strategies converge to sufficiently small neighborhood about the source. Moreover, across both stationary and dynamic source scenarios, the joint strategy exhibits not only a faster convergence to the source proximity, but also a closer approach toward the source.


These results indicate improved distance reduction performance with the joint strategy across both stationary and dynamic scenarios.

\begin{table}[t]
\centering
\caption{Terminal estimation error (m) for different source types using turn-switching directions.}
\begin{tabular}{lccc}
\hline
Source Type & \multicolumn{3}{c}{Turn-switching direction} \\
\cline{2-4}
 & \text{d\textsubscript{info}\textsuperscript{(i)}} 
 & \text{d\textsubscript{est}\textsuperscript{(i)}} 
 & \text{d\textsubscript{joint}\textsuperscript{(i)}} \\
\hline
Stationary      & 1.569 & 0.457 & \textbf{0.366} \\
Straight-line   & 1.507 & 2.033 & \textbf{0.265} \\
Circular        & 4.761 & 2.946 & \textbf{1.010} \\
\hline
\end{tabular}
\label{tab:terminal_error}
\end{table}


\section{Conclusion}
\label{sec:conclusion}

This paper presented a loop-based source-seeking framework that integrates EKF-based estimation with information-driven motion. 
The proposed joint strategy combines exploration and exploitation through an adaptive weighting, enabling reliable convergence toward the source.

The loop-based formulation enables direction updates based on aggregated information rather than noisy instantaneous data, ensuring smooth, curvature-bounded motion. This leads to stable and consistent progression toward the source.
Simulation results for both stationary and moving sources show that the joint strategy achieves lower terminal estimation error and more effective distance reduction compared to $\mathbf{d}_{\text{info}}^{(i)}$ and $\mathbf{d}_{\text{est}}^{(i)}$, particularly in dynamic scenarios.
Future work will focus on extending the framework to 3D scenarios and validating the approach through real-world experiments on agent platforms.

\bibliographystyle{IEEEtran}
\bibliography{refer}

\end{document}